%% file: root.tex
\documentclass[a4paper,conference]{IEEEtran}

\IEEEoverridecommandlockouts
\usepackage[final]{pdfpages}
\usepackage[numbers]{natbib}
\usepackage{amsmath,amssymb,amsfonts}
\usepackage{algorithm}
\usepackage{algcompatible}
\usepackage{graphicx}
\usepackage{textcomp}
\usepackage{xcolor}
\usepackage{subfig}
\usepackage{tabularx,booktabs}

\def\BibTeX{{\rm B\kern-.05em{\sc i\kern-.025em b}\kern-.08em
    T\kern-.1667em\lower.7ex\hbox{E}\kern-.125emX}}
    
\usepackage{amsmath,amsthm,amssymb,bbm,neuralnetwork,hyperref}

\theoremstyle{plain}
\newtheorem{theorem}{Theorem}[section]
\newtheorem{defi}{Definition}[section]
\newtheorem{prop}{Proposition}[section]

\newcommand{\E}{\mathbb{E}}
\newcommand{\pr}{\mathbb{P}}
\newcommand{\R}{\mathbb{R}}

\begin{document}

\title{A Multilinear Sampling Algorithm to Estimate Shapley Values}

\author{%
\IEEEauthorblockN{Ramin Okhrati}
\IEEEauthorblockA{%
\textit{University College London}\\
London, United Kingdom\\
r.okhrati@ucl.ac.uk}
\and
\IEEEauthorblockN{Aldo Lipani}
\IEEEauthorblockA{%
\textit{University College London}\\
London, United Kingdom\\
aldo.lipani@ucl.ac.uk}
}

\maketitle

\begin{abstract}
Shapley values are great analytical tools in game theory to measure the importance of a player in a game. Due to their axiomatic and desirable properties such as efficiency, they have become popular for feature importance analysis in data science and machine learning. However, the time complexity to compute Shapley values based on the original formula is exponential, and as the number of features increases, this becomes infeasible. \citet{castro2009polynomial} developed a sampling algorithm, to estimate Shapley values. In this work, we propose a new sampling method based on a multilinear extension technique as applied in game theory. The aim is to provide a more efficient (sampling) method for estimating Shapley values. Our method is applicable to any machine learning model, in particular for either multi-class classifications or regression problems. We apply the method to estimate Shapley values for multilayer perceptrons (MLPs) and through experimentation on two datasets, we demonstrate that our method provides more accurate estimations of the Shapley values by reducing the variance of the sampling statistics. 
\end{abstract}


\vspace{1.5em}
\section{Introduction}
\vspace{0.5em}
\input{introduction}

The reminder of this paper goes as follows. 
In \S~\ref{sec:related_work}, we discuss related works.
In \S~\ref{sec:background}, we introduce the background and preliminaries necessary to explain our method in \S~\ref{sec:owen_shap}. \S~\ref{sec:experiments} is devoted to the experimental setup. Results are presented and discussed in \S~\ref{sec:discussion}. 
We conclude in \S~\ref{sec:conclusion}. 

\section{Related Work}
\label{sec:related_work}
\input{related_work}

\section{Background and Preliminaries}
\label{sec:background}
\input{background}

\section{The Proposed Sampling Algorithms}
\label{sec:owen_shap}
\input{owen_shap}

\section{Experimental Setup}
\label{sec:experiments}
\input{experiments}

\section{Results and Discussion}
\label{sec:discussion}
\input{discussion}

\section{Conclusion and Future Work}
\label{sec:conclusion}
\input{conclusion}



\bibliographystyle{IEEEtranN}
\bibliography{main}



\section*{Supplementary material}

# Appendix: A Multilinear Sampling Algorithm to Estimate Shapley Values

Ramin Okhrati
*University College London*
London, United Kingdom
r.okhrati@ucl.ac.uk

Aldo Lipani
*University College London*
London, United Kingdom
aldo.lipani@ucl.ac.uk

## I. Appendix A

For the convenient of the reader, we review certain concepts and definitions of the paper. In what follows, $\nu$ represents the output of the machine learning model, $I_m^{(q)} \odot X$ is a vector obtained by the element-wise multiplications of the elements of $I_m^{(q)}$ and $X$, and $X^{(j)} := (0, 0, \ldots, 0, x_j, 0, \ldots, 0)$.

***Proof of Proposition III.1.*** It is shown by [1] that the Shapley value of feature $j$ is given by

$$S_j(v) = \int_0^1 e_j(q)\, dq, \qquad (1)$$

where

$$e_j(q) = \mathbb{E}[\nu(E_j \cup \{j\}) - \nu(E_j)], \qquad (2)$$

and $E_j$ is the random subset of $X - \{x_j\}$ equipped with the aforementioned probabilistic structure in our work. Therefore, Equation (2) can be equivalently written as

$$e_j(q) = \mathbb{E}[\nu(\mathcal{I}^{(q)} \odot X + X^{(j)}) - \nu(\mathcal{I}^{(q)} \odot X)],$$

where $\mathcal{I}^{(q)}$ is a random vector with each element admits a Bernoulli distribution with parameter $q$.

Since the mapping $q \mapsto e_j(q)$ is Riemann integrable, it can be approximated using a Riemann sum as follows:

$$S_j(\nu) = \lim_{Q \to \infty} \sum_{i=0}^{Q} \frac{1}{Q} e_j(i/Q). \qquad (3)$$

Furthermore, for a fixed value, $i$, $0 \leq i \leq Q$, $e_j(i/Q)$ can be approximated using Monte Carlo simulation, i.e.,

$$e_j(i/Q) = \lim_{M \to \infty} \frac{1}{M} \sum_{m=1}^{M} (\nu(I_m^{(i/Q)} \odot X + X^{(j)}) - \nu(I_m^{(i/Q)} \odot X)),$$

where $I_m^{(i/Q)}$, for a fixed $m$, is a vector of Bernoulli random numbers with parameter $i/Q$.

Hence, Equation (3) is equal to

$$\frac{1}{Q \times M} \lim_{Q \to \infty} \lim_{M \to \infty} \sum_{i=0}^{Q} \sum_{m=1}^{M} (\nu(I_m^{(i/Q)} \odot X + X^{(j)}) - \nu(I_m^{(i/Q)} \odot X)). \qquad (4)$$

Now, let us define the estimator $\hat{S}_j(\nu)$ by

$$\hat{S}_j(\nu) := \frac{1}{Q \times M} \sum_{i=0}^{Q} \sum_{m=1}^{M} (\nu(I_m^{(i/Q)} \odot X + X^{(j)}) - \nu(I_m^{(i/Q)} \odot X)).$$

This is the same estimator used in Algorithm 1 which converges to $S_j(\nu)$ by Equation (4). Therefore, we have proved that as $Q$ and $M$ increase $\hat{S}_j(\nu)$ converges to $S_j(\nu)$.
$\square$

***Proof of Theorem III.1.*** The proof follows from the same methodology as in the proof of Proposition III.1 which is based on Equation (1). This time, the idea is to write Equation (1) as follows:

$$S_j(v) = \int_0^1 e_j(q)\, dq = \int_0^{0.5} e_j(q)\, dq + \int_{0.5}^{1} e_j(q)\, dq$$
$$= \int_0^{0.5} (e_j(q) + e_j(1-q))\, dq.$$

Now, similar to the proof of Proposition III.1, this integral can be approximated using the Riemann sum:

$$S_j(\nu) = \lim_{Q \to \infty} \sum_{i=0}^{Q/2} \frac{1}{Q} (e_j(i/n) - e_j(1 - i/n)). \qquad (5)$$

For a fixed value $0 \leq i \leq Q$, $e_j(i/n)$ and $e_j(1-i/n)$ can be approximated respectively using the Monte Carlo simulations

$$e_j(i/Q) = \lim_{M \to \infty} \frac{1}{M} \sum_{m=1}^{M} (\nu(I_m^{(i/Q)} \odot X + X^{(j)}) - \nu(I_m^{(i/Q)} \odot X)),$$

and

$$e_j(1 - i/Q) = \lim_{M \to \infty} \frac{1}{M} \sum_{m=1}^{M} (\nu(I_m^{(1-i/Q)} \odot X + X^{(j)}) - \nu(I_m^{(1-i/Q)} \odot X)).$$

Therefore, Equation (5) can be written as

$$S_j(\nu) = \frac{1}{Q \times M} \lim_{Q \to \infty} \lim_{M \to \infty} \sum_{i=0}^{Q/2} \sum_{m=1}^{M} (h_{i/Q}^{j,m} + h_{1-i/Q}^{j,m}),$$

where
$$h_q^{j,m} = \nu(I_m^{(q)} \odot X + X^{(j)}) - \nu(I_m^{(q)} \odot X).$$

Finally, we define the estimator $\hat{S}_j(\nu)$ by
$$\hat{S}_j(\nu) = \frac{1}{Q \times M} \sum_{i=1}^{Q/2} \sum_{m=1}^{M} (h_{i/Q}^{j,m} + h_{1-i/Q}^{j,m}).$$

This is the same estimator used in Algorithm 2. So we have just proved that as $Q$ and $M$ increase, $\hat{S}_j(\nu)$ covnerges to $S_j(\nu)$.
□



\end{document}

%% file: introduction.tex
Explaining models have becoming an important aspect of data science, in particular, with the growing complexity of machine learning models. 
For motivations, discussions, and different explainability models under a unified framework, we refer to the work of \citet{lundberg2017unified} where the methods LIME \cite{ribeiro2016should}, DeepLIFT \cite{shrikumar2017learning}, Layer-Wise Relevance Propagation \cite{bach2015pixel}, and exact Shapley values \cite{shapley1953value} together with their SHAP methods are explained. 
The main reason for the popularity of Shapley values is that they satisfy certain interesting axioms (which we will discuss in Section \ref{sec:background}), most notably the efficiency axiom, that are desirable in analyzing interpretability and attributions.

It is worth noting that Shapley values are a local interpretability attribution method in the sense that they quantify the explanation of a specific model by assigning a value to each input feature. This value can be then interpreted as the contribution (or relevance) of that feature towards the output of the model. 
Although several non-axiomatic methods are developed in particular for neural networks, an advantage of an axiomatic approach, such as Shapley values, is to have a profound theoretical background that mitigates the risk of misleading or unreliable interpretations. While there are other attribution methods which have their own merits, we specifically focus on estimating Shapley values through sampling methods. 

The main drawback of the calculation of Shapley values using the original definition (aka \textit{exact Shapely values}) is its exponential time complexity. We divide Shapley values estimation methods into two main categories, methods that do not converge towards the exact Shapely values, and methods that do converge.
%
In the first category, we have two types of methods. The first type are semi-closed form solution methods to approximate Shapley values which provide a decent approximation for simple games (see the work of \citet{castro2009polynomial} where they make a comparison of this method with their sampling approach). These solutions are normally based on the central limit theorem (see the work of \citet{owen1972multilinear}) which are computationally cheap. 
The second type of the first category are data-driven methods such as regression and linear based techniques \cite{lipovetsky2001analysis,fatima2008linear}, quantitative input influence approaches \cite{datta2016algorithmic}, and DASP \cite{ancona2019explaining}. 
The second category includes the Shapley sampling method developed by \citet{castro2009polynomial} which in general becomes computationally expensive as the number of feature increases, however, this has the advantage of converging to the exact Shapley values. 

Our method falls within the second category in which we provide a sampling type algorithm to estimate Shapley values based on a multilinear extension technique as applied in game theory \cite{owen1972multilinear}, which converges to the exact Shapley values. 
As mentioned earlier, semi-closed form solutions have already been obtained from this multilinear extension techniques for simple games \cite{owen1972multilinear}, however, to the best of our knowledge, this is the first time that a sampling type algorithm is extracted from the multilinear extension techniques. Furthermore, our algorithm is applicable for calculating the Shapley values of features of any machine learning model. 
In particular, we apply the algorithm for the case of MLPs in classification problems. 

The sampling methods such as the one developed by \citet{castro2009polynomial} are used as a ground truth for Shapley value estimations or attributions, see for instance the work of \citet{ancona2019explaining}. Their experiments show that for up to a certain number of network evaluations, their method provides a more accurate solution. In this paper we compare against the sampling algorithm developed by \citet{castro2009polynomial} and from now on we will refer to this algorithm as \textit{Castro sampling}. 
%

As our experiments show, comparing to the Castro sampling, our method reduces the variance of the estimator which results in a more accurate and time-efficient algorithm. The goal is to present our algorithm in its original form though it could be improved in several directions which we leave to future work. In our experimental analysis, first, we use a dataset with a small number of features in which we can calculate the exact Shapley values. Knowing the exact Shapley values in this case, we can then compare the quality of our method, in terms of both accuracy and time-efficiency against Castro sampling. As the results confirm, the error of our approximation is much lower. Next, we then apply the method also on a dataset with a larger number of features. While it is not possible to obtain the exact Shapley values in this case, we can perform a variance analysis to show that our estimator admits a lower variance compared to that of Castro sampling. 




%% file: related_work.tex
 
 Multilinear extension methods have been already applied in calculating feature attributions. \citet{jones2010multilinear} present a multilinear extension technique, which is then applied to multichoice games in the context of Shapley values using the central limit theorem.  \citet{michalak2013efficient} present analytical formulas for Shapley value-based centrality in both weighted and unweighted networks algorithms. 
 The base of our method is the multilinear extension technique as developed by \citet{owen1972multilinear}. 
 This work is further discussed in the next section. Despite other multilinear extension techniques applied to estimate Shapley values, our method does indeed converge to the exact Shapley values. Although sampling methods should converge to the exact values, the sampling nature still remains there, and it becomes computationally cumbersome in the presence of a large number of features. So the efficiency and applicability of the method boils down to its convergence speed. 
 
 Depending on the structure of the machine learning model, in certain cases, simpler approximations might be possible.  \citet{fatima2008linear} introduce a polynomial time approximation method to estimate Shapley values for weighted voting games. A similar idea is also used by \citet{ancona2019explaining} for explaining deep neural networks where in contrast to the sampling method developed by \citet{castro2009polynomial}, their algorithm does not converge to the exact Shapley values, however, for up to a fixed number of instructions, their algorithm provides a more accurate estimation. 

It is worth noting that there are other attribution methods as well. \citet{lundberg2017unified} introduce the KernelSHAP method, which combines sampling with lasso regression. 
There are also Integrated Gradient \cite{sundararajan2017axiomatic}, and DeepLIFT and its two flavors Rescale and RevealCancel \cite{shrikumar2016not,shrikumar2017learning}. However, among all these methods, only KernelSHAP (without regularizer) converges to the exact Shapley values. 


%% file: background.tex
\subsection{Feature Attribution Through Shapely Values}

Consider a machine learning model $\nu$ from  $\R^{n+1}$ into $\R$ with inputs represented by $X=(x_0,x_1,...,x_n)$ in the feature space $\R^{n+1}$, where $x_0=1$ is to account for the bias term, and {$n\geq1$ is a positive integer.}


A prime question in many applied fields such as data science and finance is to provide an algorithm that measures the contribution (also called attribution or importance) of a given feature $x_j$, $j=0,1,...,n$, in the output $\nu(X)$ where $X$ is the input of the machine learning model. But how should such notion of importance be defined? 



Let us have a short excursion to game theory where the notion of contribution is well defined for games. In game theory language, a game is a real-valued function $\xi$ defined on the player set (called features in our setup) $\{x_0,x_1,\dots,x_n\}$, i.e., $\xi:2^{\{x_0,x_1,\dots,x_n\}}\rightarrow \R$. In contrast to machine learning where the inputs are tuples (vectors) of features, in game theory the inputs are set of players. 

Suppose that we want to define a notion of contribution for the player $x_j$, $j=0,1,...,n,$ of a game expressed by $\xi$. Let us denote the contribution of feature $x_j$ by $\Psi_j(\xi)$. It is argued by \citet{shapley1953value} that $\Psi_j(\xi)$ is a reasonable notion of contribution, if it satisfies (at least) the following axioms:
\begin{LaTeXdescription}
\vspace{0.5em}\item[\textbf{Efficiency}:]
${\sum_{j=0}^n\Psi _{j}(\xi)=\xi(X)}$. That is, the sum of all contributions is equal to the output of the game.
\vspace{0.5em}\item[\textbf{Symmetry}:] For any two players $x_i$ and $x_j$, if $\xi(A\cup \{x_i\})=\xi(A\cup \{x_j\})$ for all subset $A$ (of the player set) that do not contain neither $x_i$ nor $x_j$, then $\Psi_j(\xi)=\Psi_i(\xi)$. That is, if two players have the same influence on any coalition their contributions should be the same.
\vspace{0.5em}\item[\textbf{Linearity}:] For two games $\xi$ and $\omega$ with the same player set and for all players $x_j$ and a real number $\alpha$, we have  $\Psi_j(\xi+\omega)=\Psi_j(\xi)+\Psi_j(\omega)$ and $\Psi_j(\alpha \xi)=\alpha \Psi_j(\xi)$. That is, when two independent games are combined, the total contribution of a player is equal to the sum of individual contributions on each game.
\vspace{0.5em}\item[\textbf{Null player (feature)}:] For a player $x_j$ and a game $\xi$, if $\xi(A\cup \{x_j\})=\xi(A)$ for all subsets $A$ (of the player set) that do not contain $x_j$ (such player is called null player) then $\Psi_j(\xi)=0.$
\end{LaTeXdescription}

Remarkably, it is proved by \citet{shapley1953value} that $\Psi_j(\xi)$ must satisfy the following representation:
\begin{equation}
    \label{eq:shapley}
    \Psi_{j}(\xi)=\sum _{A\subseteq X\setminus \{x_j\}}{\frac {|A|!\;(n-|A|)!}{(n+1)!}}\left(\xi(A\cup \{x_j\})-\xi(A)\right),
\end{equation}
where $|A|$ is the cardinality  of the set $A$, i.e., the number of elements in $A$. In other words, the only measurement that satisfies the four axioms is given by Equation \eqref{eq:shapley}, which is called the Shapley value (see  the work of \citet{moulin2004fair} for more explanations). 

Given the desirable axioms, it appears that Shapley values are suitable measurements of a player's contribution in a game which make them also suitable candidates for measuring  contributions of features in machine learning.  Concretely, a machine learning model $\nu$ on a feature space (player set in game theory) can be considered as a game, and a tuple of features $X=(x_0,x_1,\dots,x_n)$ is interpreted as the set of players $\{x_0,x_1,\dots,x_n\}$. Vice versa, for a set of players (each player is interpreted as a feature), we can create a tuple where a missing feature is substituted by zero. The last point is related to the concept of baseline in machine learning where we substitute a missing feature with a zero baseline. Other choices of baselines are also available \cite{izzo2020baseline}, but this does not affect our methodology.

From the above argument, we can formally define the contribution of a feature $x_j$ in a machine learning model $\nu$ through Shapley value as follows:
\begin{defi}
Given a feature $x_j$ and a machine learning model $\nu$, we define its contribution towards $\nu(X)$ as the Shapley value of the feature $x_j$ for $\nu$, that is: 
$$S_{j}(\nu)=\sum _{A\subseteq X\setminus \{x_j\}}{\frac {|A|!\;(n-|A|)!}{(n+1)!}}(\nu(A\cup \{x_j\})-\nu(A)),$$
where $|A|$ is the cardinality  of the set $A$, and with some abuse of notation, $\nu(A\cup\{x_j\})$ and $\nu(A)$ must be understood as the evaluation of $\nu$ for the corresponding tuples obtained respectively from $A\cup\{x_j\}$ and $A$, through replacing a missing feature by zero in the tuples. 
\end{defi}

\subsection{Multilinear Extensions Method}

Next, we provide a brief review of the multilinear extension method as developed by \citet{owen1972multilinear}. In order to do so, first, we need to impose a probabilistic structure on the feature space in the following sense. For a fixed $j=0,1,...,n$ and a random subset $E_j$ of $\{x_0,x_1,...,x_n\}\smallsetminus\{x_j\}$, the probability of any feature (except $x_j$) being in $E_j$ is equal to $q$ where $0\leq q \leq 1$; more precisely, this means that: 
$$\pr(\{x_k\in E_0\})=q,\text{ for all } k=1,2,...,n,$$
$$\pr(\{x_k\in E_n\})=q,\text{ for all } k=0,1,...,n-1,$$
and for $j=1,2,...,n-1$,
$$\pr(\{x_k\in E_j\})=q,\text{ for all } k=0,1,...,j-1,j+1,...,n.$$
Furthermore, we assume that sampling is done so that the events $\{x_k\in E_j\}_{k\neq j, 0\leq k\leq n}$ are mutually independent. Note that because $E_j$ is a random subset, $\nu(E_j)$ becomes a random variable. 


Given the above probabilistic structure, a probabilistic representation of the Shapley values is provided as follows \cite{owen1972multilinear}: 
\begin{equation}\label{eq:owen}
S_j(v)=\int_0^1 e_j(q) \;dq,
\end{equation}
where
\begin{equation}\label{eq:ejq}
    e_j(q)=\E[\nu(E_j\cup\{x_j\})-\nu(E_j)].
\end{equation}
The integral of Equation \eqref{eq:owen} can be interpreted as either Riemann or Lebesgue integral. 
Equations \eqref{eq:owen} and \eqref{eq:ejq} are the result of the multilinear extension method and provide the bases of the two algorithms that we present in the next section.

%% file: owen_shap.tex
In this section, we present two algorithms based on  the multilinear extension method developed by \citet{owen1972multilinear} (in his honor, we call the algorithms \textit{Owen sampling}). 

\subsection{The Owen Sampling Algorithm}

The first algorithm is based on a discretization of Equation \eqref{eq:owen}. In this algorithm, the number of samples is controlled by two positive integers $Q$ and $M$, and the Shapley values estimators are shown by $\hat{S}(\nu) = (\hat{S}_0(\nu),...,\hat{S}_n(\nu))$. Algorithm \ref{alg:alg1} explains the method in pseudo-code.



\begin{algorithm}[!t]
    \caption{Owen Sampling}\label{alg:alg1}
    \hspace*{\algorithmicindent} \textbf{Input:} $\nu$, $X=(x_0,x_1,\dots,x_n)$, $Q$, $M$\\
    \hspace*{\algorithmicindent} \textbf{Output:} $\hat{S}(\nu) = (\hat{S}_0(\nu),\hat{S}_1(\nu),\dots,\hat{S}_n(\nu))$\\ 
    \begin{algorithmic}[1]
        \STATE $\hat{S}_j(\nu):=0$, for all $j$
        \FOR {$q=0,1/Q,2/Q,\dots,1$}
            \STATE $e_j:=0$, for all $j$
			\FOR {$m=1,2,\ldots,M$}
				\STATE $I^{(q)}_m \gets (b_j : b_j \sim \text{Bern}(q), \;\text{for all} \; j)$
				\STATE $h^{(q)}_{m,j} \gets \nu(I^{(q)}_m\odot X+X^{(j)} )-\nu(I^{(q)}_m\odot X), \text{for all} \, j$
				\STATE $e_j \gets e_j+h^{(q)}_{m,j}, \;\text{for all} \; j$
			\ENDFOR
			\STATE $\hat{S}_j(\nu) \gets \hat{S}_j(\nu)+ e_j, \;\text{for all} \; j$
		\ENDFOR
		\STATE $\hat{S}(\nu) \gets \hat{S}_j(\nu)/(QM), \;\text{for all} \; j$
    \end{algorithmic}
\end{algorithm}

This Algorithm requires as input $X$, $\nu$, and the number of samples which are controlled by the parameters $Q$ and $M$. The parameter $Q$ controls the level of discretization of the integral in Equation \eqref{eq:owen}, and the parameter $M$ controls the accuracy of the estimation of the expected value $e_j(q)$ in Equation \eqref{eq:ejq}. The optimal choices of $Q$ and $M$ are yet to be explored, however, of the two, the parameter $Q$ plays a more significant role in the sense that by increasing $Q$, in general, we increase the total random numbers generated in the inner loop controlled by $M$. But the opposite of this does not occur. In all our experiments, we fix the value of $M=2$, i.e., the total number of samples is equal to $2Q$.

The outer loop in Line 2 of this algorithm approximates the integral of Equation \eqref{eq:owen} by initializing zero values for its integrands $e_j$ in Line 3. If we divide these values by $M$, i.e. $e_j/M$, they are approximations of $e_j(q)$ in Equation \eqref{eq:ejq}. More precisely, the inner loop in Line 4 uses a Monte Carlo estimation method that leads to the estimation of Equation \eqref{eq:ejq}. Since the probabilistic structure of the algorithm is governed by Bernoulli distribution, in Line 5, for each iteration of the inner loop, we generate $n+1$ Bernoulli random numbers independently, which are temporarily stored in the vector $I^{(q)}_m$, i.e., each element of this vector is either zero or one depending on the outcome of the Bernoulli random number.

This vector is used to calculate the marginal contributions $h^{(q)}_{m,j}$. In Line 6, the $\odot$ operator is the element-wise multiplication, and it is applied to the vectors $I^{(q)}_m$ and $X$. The operand $X^{(j)}$ in Line 6  indicates the vector $(0,0,\dots,0,x_j,0,\dots,0)$.
The unnormalized approximated values of the integral in Equation \eqref{eq:owen} are then updated in Line 9. To normalize them we need to divided them by the total number of iterations $QM$. The convergence of this algorithm is guaranteed by the following proposition.

\begin{prop}
\label{prop:prop1}
Suppose that for a fixed $j$, function $q\mapsto e_j(q)$ is Riemann (or Lebesgue) integrable\footnote{This is a technical assumption, and it is satisfied in all practical circumstances.}, i.e., $\int_0^1 |e_j(q)|\; dq<\infty$. Then the estimators of Algorithm \ref{alg:alg1} converge to the true Shapley values as $Q$ and $M$ increase.
\end{prop}
\begin{proof}
See the Appendix. 
\end{proof}

\subsection{Time Complexity Analysis}

Algorithm \ref{alg:alg1} can be improved using parallel programming, however, an implementation method such as this is not our main criterion of comparison. Note that generating a Bernoulli random number is computable in constant time, i.e., $O(1)$ \cite{bringmann2013exact,kachitvichyanukul1988binomial}. Therefore, assuming that the time complexity of calculating $h^{(q)}_{m,j}$ is polynomial, the time complexity of the algorithm is polynomial. This has the same time complexity of the Castro sampling algorithm \cite{castro2009polynomial}, but in order to have a fair comparison, a more detailed investigation is required. 

In the Castro sampling algorithm, in each iteration three major operations are performed: 
\begin{enumerate}
    \item first, a random permutation from the set of $n$ 
    features is drawn using a uniform probability, the time-complexity of this is $O(n)$;
    \item second for a given permutation, the set of predecessors of the features in that permutation is calculated. This is also $O(n)$, and;
    \item finally, the value of the machine learning model is calculated.
\end{enumerate}
By a one to one comparison of Castro sampling algorithm to Algorithm \ref{alg:alg1}, in each iteration, we draw a random sample with time complexity 
$n \cdot O(1) = O(n)$, instead of calculating the predecessors of the features, we calculate the identity $e_j=e_j+h_{m,j}^{(q)}$, for all $j$, that is $O(n)$, and finally we calculate the value of the machine learning model. 
In Algorithm \ref{alg:alg1}, a total of $QM$ samplings is required, while in the Castro sampling algorithm, this is handled by a single parameter, $M_c$. 

Therefore, for a fair comparison, for $QM$ samplings of Algorithm \ref{alg:alg1}, we have to consider $M_c=QM$ ones of the Castro sampling algorithm. 
However, as our experiments show, both the accuracy and the actual running time of Algorithm \ref{alg:alg1} are better than that of the Castro sampling algorithm. An important feature of Algorithm \ref{alg:alg1} is symmetry which can be exploited to improve the Algorithm \ref{alg:alg1} as follows.

\subsection{Improving Owen Sampling through Symmetry}

For $0\leq q\leq 1$, we let $h^{(-q)}_{m,j}$ represent the calculation of the output of the machine learning model as in Line 6 of Algorithm \ref{alg:alg1}, when instead of $q$, we use $1-q$. Looking at Algorithm \ref{alg:alg1}, it is clear that by using this symmetry, we only need to run the outer loop of this algorithm halfway through, i.e., for $q=0,1/Q,...,0.5$, as for each $q$, we can duplicate the calculations of the output of the machine learning model by using $1-q$. 
For instance, suppose that we only have three features $(x_0,x_1,x_2)$, and for a given $q$, we have generated the Bernoulli random sequence $(1,1,0)$ which means that the features $x_0$ and $x_1$ are selected according to this distribution, and feature $x_2$ is absent. This is  equivalent to choosing only $x_2$ based on the random sequence $(0,0,1)$ generated using 
a Bernoulli distribution with parameter $1-q$. This will have two important consequences, first, we only need to generate $MQ/2$ random numbers each with Bernoulli distribution and, second and most importantly, this will further reduce the variance of the estimator which will be discussed in our experimental section. For this reason, we call this algorithm \textit{Halved Owen Sampling}. This is formalized in Algorithm \ref{alg:alg2}.

\begin{algorithm}[!t]
    \caption{Halved Owen Sampling}\label{alg:alg2}
    \hspace*{\algorithmicindent} \textbf{Input:} $\nu$, $X=(x_0,x_1,\dots,x_n)$, $Q$, $M$\\
    \hspace*{\algorithmicindent} \textbf{Output:} $\hat{S}(\nu) = (\hat{S}_0(\nu),\hat{S}_1(\nu),\dots,\hat{S}_n(\nu))$\\ 
    \begin{algorithmic}[1]
        \STATE $\hat{S}_j(\nu):=0$, for all $j$
        \FOR {$q=0,1/Q,2/Q,\dots,0.5$}
            \STATE $e_j:=0$, for all $j$
			\FOR {$m=1,2,\ldots,M$}
				\STATE $I^{(q)}_m \gets (b_j : b_j \sim \text{Bern}(q), \;\text{for all} \; j)$
				\STATE $I^{(-q)}_m \gets 1 - I^{(q)}_m$
				\STATE $h^{(q)}_{m,j} \gets \nu(I^{(q)}_m\odot X+X^{(j)} )-\nu(I^{(q)}_m\odot X), \text{for all} \, j$
				\STATE $h^{(-q)}_{m,j} \gets \nu(I^{(-q)}_m\odot X+X^{(j)} )-\nu(I^{(-q)}_m\odot X), \text{for all} \, j$
				\STATE $e_j \gets e_j + h^{(q)}_{m,j} + h^{(-q)}_{m,j}, \;\text{for all} \; j$
			\ENDFOR
			\STATE $\hat{S}_j(\nu) \gets \hat{S}_j(\nu)+ e_j, \;\text{for all} \; j$
		\ENDFOR
		\STATE $\hat{S}(\nu) \gets \hat{S}_j(\nu)/(QM), \;\text{for all} \; j$
    \end{algorithmic}
\end{algorithm}

%

Algorithm \ref{alg:alg2} is almost identical to Algorithm \ref{alg:alg1} except for few lines. As already discussed, comparing Line 2 of Algorithm \ref{alg:alg1} against Line 2 of Algorithm \ref{alg:alg2}, we now only need to explore $Q/2$ iterations. In Lines 7 and 8, we calculate the output of the model for two cases, $q$ and $1-q$, using the vectors created in Lines 5 and 6. The result of Line 7 and 8 is then added to Line 9. The convergence of this algorithm is guaranteed by the following theorem.

\begin{theorem}
\label{th:owen2}
Suppose that for a fixed $j$, function $q\mapsto e_j(q)$ is Riemann (or Lebesgue) integrable, i.e., $\int_0^1 |e_j(q)|\;dq<\infty$. Then the estimates of Algorithm \ref{alg:alg2} converge to the true Shapley values as $Q$ and $M$ increase.
\end{theorem}

\begin{proof}
See the Appendix.
\end{proof}

For a fair comparison of Algorithm \ref{alg:alg2} against the Castro sampling algorithm, we use similar arguments as done for Algorithm \ref{alg:alg1} where for a $QM$ sample size of this algorithm, we consider an equivalent $M_c=QM$ sample size of the Castro sampling algorithm.
This may sound counter-intuitive because it would appear that this new algorithm uses $QM/2$ samples, i.e., we generate $QMn/2$ random numbers in contrast to $QMn$ random numbers of Algorithm \ref{alg:alg1}. However,
in its inner most loop we calculate the output of the machine learning model twice. Hence, the number of operations is still in the order of $QM$. Therefore 
we consider the same sample size as for Algorithm \ref{alg:alg1}.



%% file: experiments.tex
In this section, we perform several experiments to test our algorithms. Although, the algorithms are applicable to any machine learning model, here we focus only on Multilayer Perceptrons (MLPs); further experiments on more complex networks are left to future work. The software used to run these experiments is available at the following link \url{https://github.com/aldolipani/OwenShap}.

\subsection{Datasets and Preprocessing}

We experiment with the following two datasets:

\textbf{Credit Card Dataset (CC)}. 
This is a financial dataset which is a collection of credit card data used by \citet{yeh2009comparisons}. There is a total number of 29,351 observations where each observation is made of 23 features and a binary target variable. The features are either financial (such as pay related information) or non-financial like age. The target variable is either zero or one with one indicating the default of the credit card account.

\textbf{Modified NIST (MNIST).} 
This is a large database of handwritten digits that is commonly used for training and testing machine learning models \cite{lecun1998gradient}. It was created by ``re-mixing'' the samples from another dataset. Each sample is a black and white image of a handwritten digit. Furthermore, the black and white images are normalized to fit into a 28x28 pixel bounding box. The MNIST dataset contains 70,000 images. 

These two datasets have a different number of features (23 and 784). In order to compare the various sampling algorithms we need to compute the exact Shapely values for at least one dataset. However, both datasets make this calculation infeasible due to their feature sizes. For this reason, we reduce the dimensionality of the CC dataset by only selecting the first 15 features. 





\subsection{MLPs and Training Details}

The activation function of the input layer and hidden layers is always a sigmoid, while the activation function of the output layer is a softmax. 
All MLPs are trained using Adam as optimizer with its default parameters and the binary cross-entropy as loss function.
We use an early stopping criteria to avoid overfitting, where we stop training after the loss on the validation set has not improved over 3 epochs. 
We divide both datasets into 64\% training set, 16\% validation set and 20\% test set.
The rest of the hyper-parameters, i.e., the number of hidden layers and number of neurons in each layer, are chosen via a Monte Carlo sampling of models using the training and validation sets. 

For the CC dataset, we sample 1000 models with a number of hidden layers from 0 to 3 and a number of neurons for each hidden layer from 1 to 15. 
The best model has 2-hidden layers with 13 and 9 neurons each. The accuracy of this model on the CC test set is 0.8247.
For the MNIST dataset, we sample 1000 models with a number of hidden layers from 0 to 3 and a number of neurons for each hidden layer from 25 to 500 at multiples of 25.
The best model has 2-hidden layers with 300 and 25 neurons each. The accuracy of this model on the MNIST test set is 0.9818.

\subsection{Experiments}

\textbf{Approximation Error.}
To evaluate the quality of the different sampling algorithms, we compare them against the exact Shapely values. 
Computing the Shapely values is an expensive operation because it grows exponentially with the number of features. In order to make this feasible, we only perform this analysis on the smaller CC dataset. 
Here, we take a sample of size 50 of the test set and compute the exact Shapely values. Then for each algorithm, as we increase the number of algorithms' samplings, we compute the average of the Mean Squared Errors (MSEs) across the examples. 
In order to make a fair comparison, so that the numbers of samples of the various algorithms are equivalent, we set $M_c = 2m$ for the Castro sampling, where $m$ is a positive integer, and $Q = m$ and $M=2$ for the Owen sampling. 

\textbf{Variance Analysis.}
Besides the quality of the predictions, we perform a complementary analysis of the sampling algorithms by measuring their sample variances. To do this, we first take a sample of the test set of size 50. Then for each sampling algorithm and each example, we estimate 
the Shapely values increasing the sampling sizes from 2 to 200 at steps of 2. 
At each step, we compute the standard deviations of the feature estimates over the previous steps. 
Then, we compute the average across the features. 
Finally, we average across the examples.
Note that the increase in the sample size is done by varying the equivalent number of samples, i.e., $M_c=2Q$. 
As the number of samples increases, a decrease of the averaged sample standard deviations would confirm the convergence of the estimators. 
Note that from Proposition \ref{prop:prop1} and Theorem \ref{th:owen2}, the Owen sampling algorithms converge to the exact Shapley values only when both $M$ and $Q$ increase. By fixing $M=2$, the algorithms converge to an estimation of the Shapley values.

%% file: discussion.tex
\begin{figure*}[!t]
	\centering
	\subfloat{\includegraphics[width=0.495\linewidth]{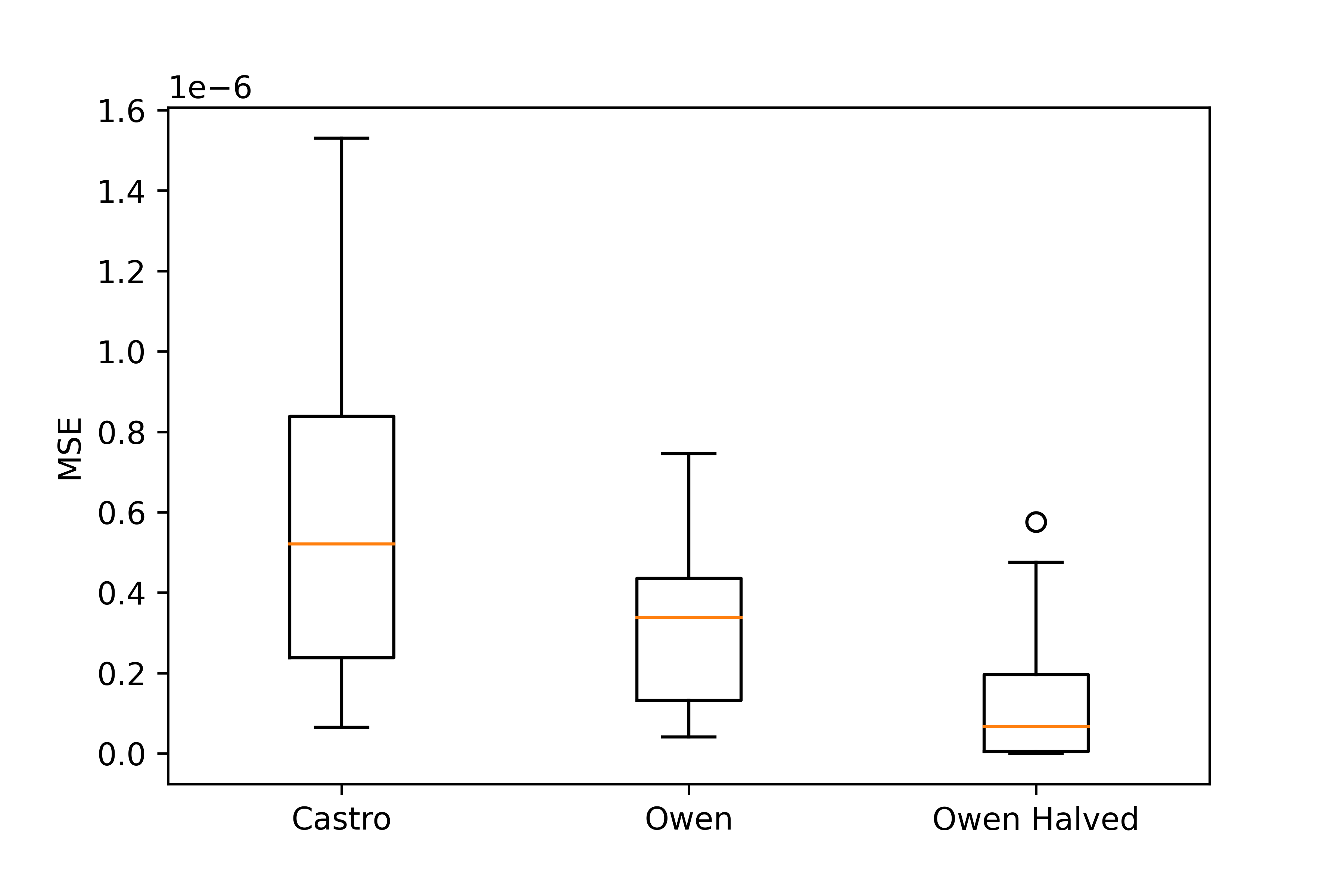}}
	\hfill
	\subfloat{\includegraphics[width=0.495\linewidth]{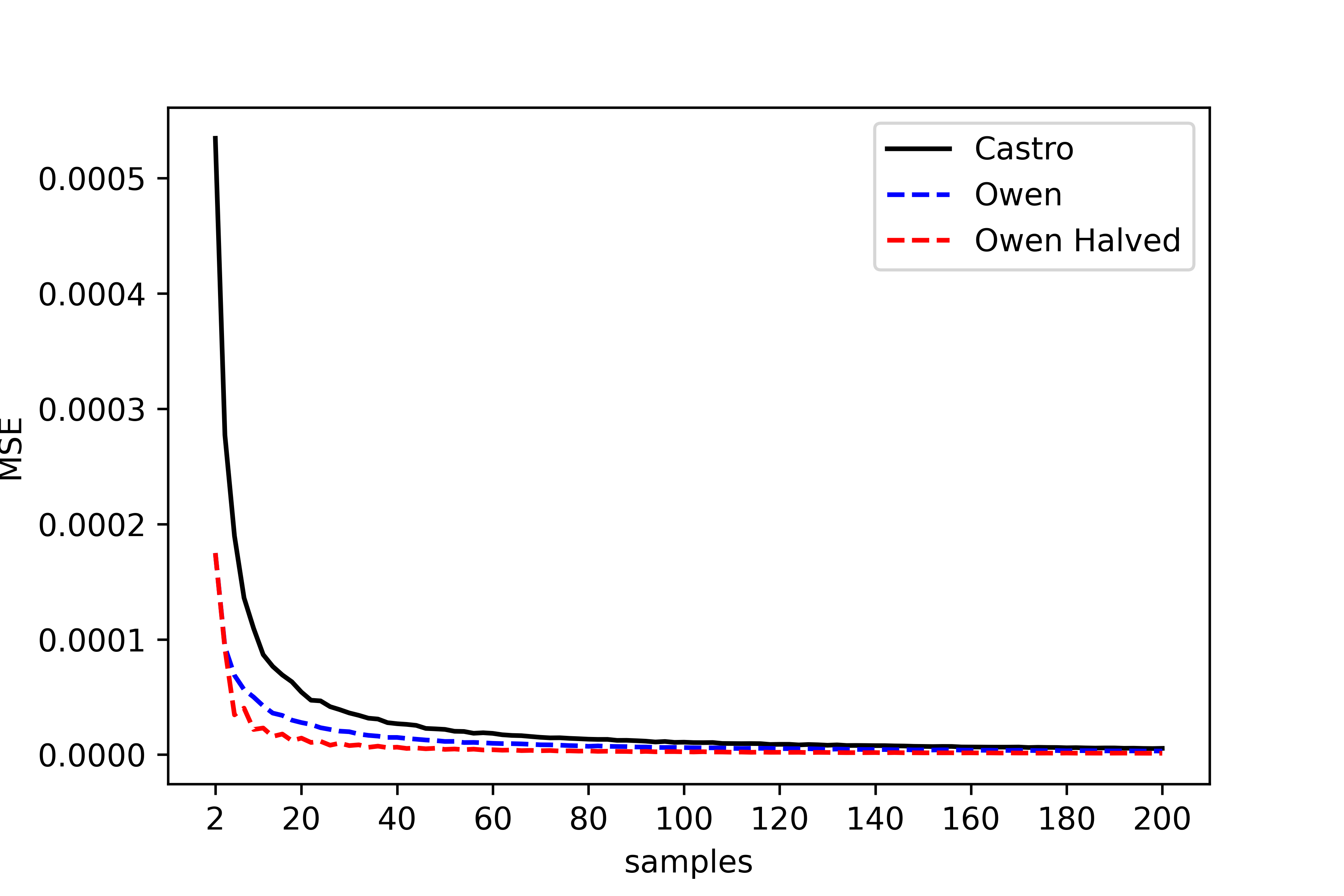}}
    \caption{MSE of the exact Shapely values against the Castro and Owen samplings. The box-plot shows the variance of the error over 50 examples and 2000 samples. The line-plot shows how the error decreases when the number of samples is increased.}
	\label{fig:error_credit}
	\vspace{1em}
\end{figure*}

\begin{table}[!t]
    \centering
    \begin{tabular}{r|c|c|c}
        \toprule
        Algorithm  & Parameters        & MSE ($10^{-6})$ & Time (ms) \\
        \hline
        Castro              & $M_c=2000$          & 0.5575 & 3.004 \\             
        Owen                & $M=2, \; Q=1000$  & 0.3184 &  1.044 \\
        Halved Owen         & $M=2, \; Q=1000$  & 0.1207 &  0.968 \\
        \bottomrule
    \end{tabular}
    \caption{Summery of the results presented in box-plot in Fig.~\ref{fig:error_credit}. Results are averaged over 50 examples of the CC test set.}
    \label{tab:summary_boxplot}
\end{table}

In Figure \ref{fig:error_credit}, we observe the results of the error analysis. Both plots measure the MSE of the exact Shapely values against the outcome of the sampling algorithms. The box-plot is computed over 50 examples from the test set and using 2000 equivalent samples ($M_c = 2000$ for Castro, $Q=1000$ and $M=2$ for Owen). From this box-plot we can observe that the Owen sampling outperforms the Castro sampling. Moreover, the Halved variant of the Owen sampling further improves the original Owen sampling. In Table \ref{tab:summary_boxplot}, we observe the summary of the results shown in the box-plot. Moreover, we also present the expected running-time per example of the algorithms. This result, although may depend on the CPU  (AMD Threadripper 2950X CPU) and GPU (Nvidia Titan RTX) used, shows that the Owen algorithm achieved around a 3-fold speed-up against the Castro algorithm.

In Figure \ref{fig:error_credit}, we also graph a line-plot of the MSEs computed over 50 examples from the CC test set. The x-axis of this plot represents the number of equivalent samples. In this plot we observe that the MSE of the Owen sampling tends to zero faster than that of Castro sampling. 
This plot further confirms the superiority of the Owen sampling and in particular of its Halved version.

\begin{figure*}[!ht]
    \vspace{1.5em}
	\centering
	\subfloat{\includegraphics[width=0.495\linewidth]{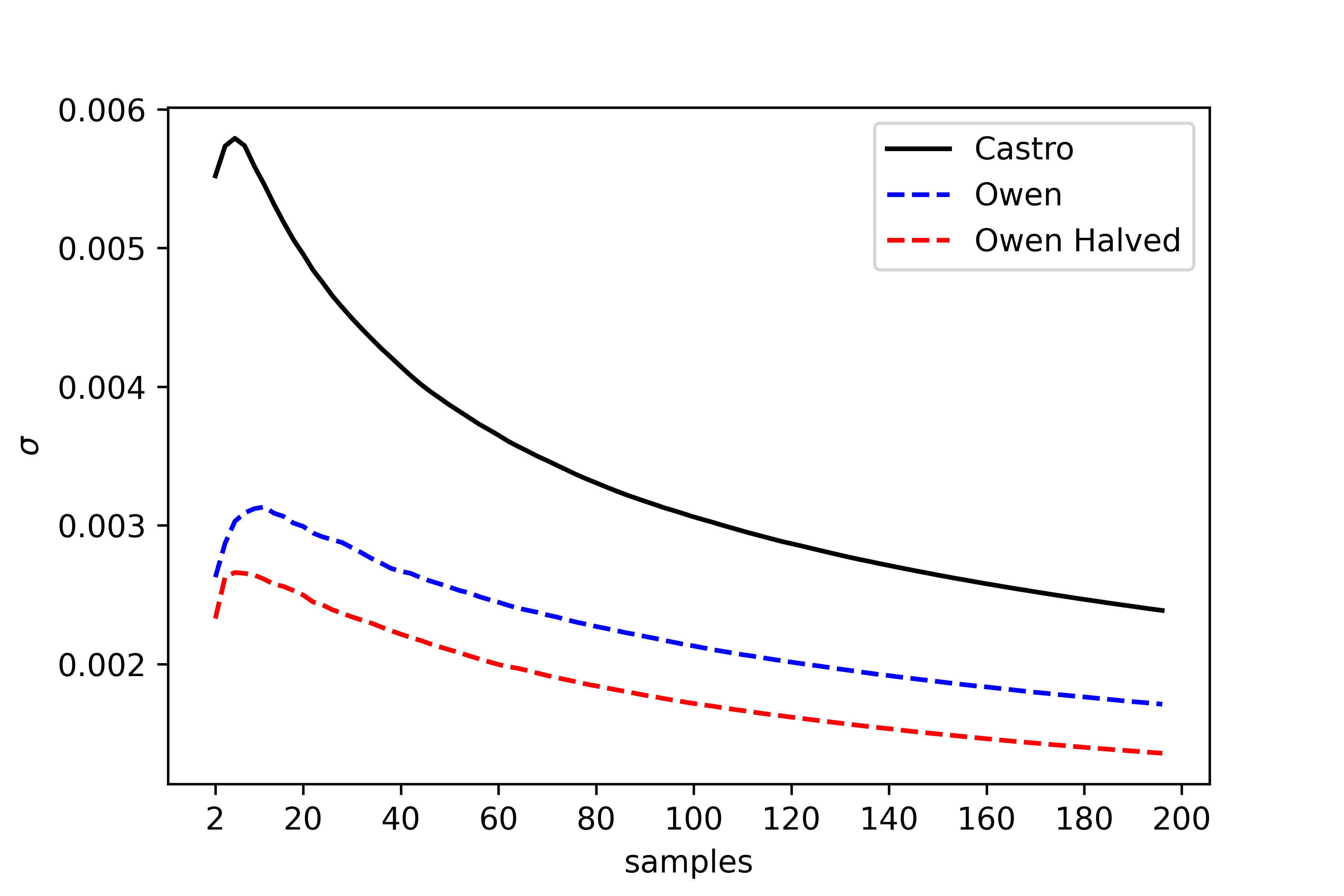}}
	\hfill
	\subfloat{\includegraphics[width=0.495\linewidth]{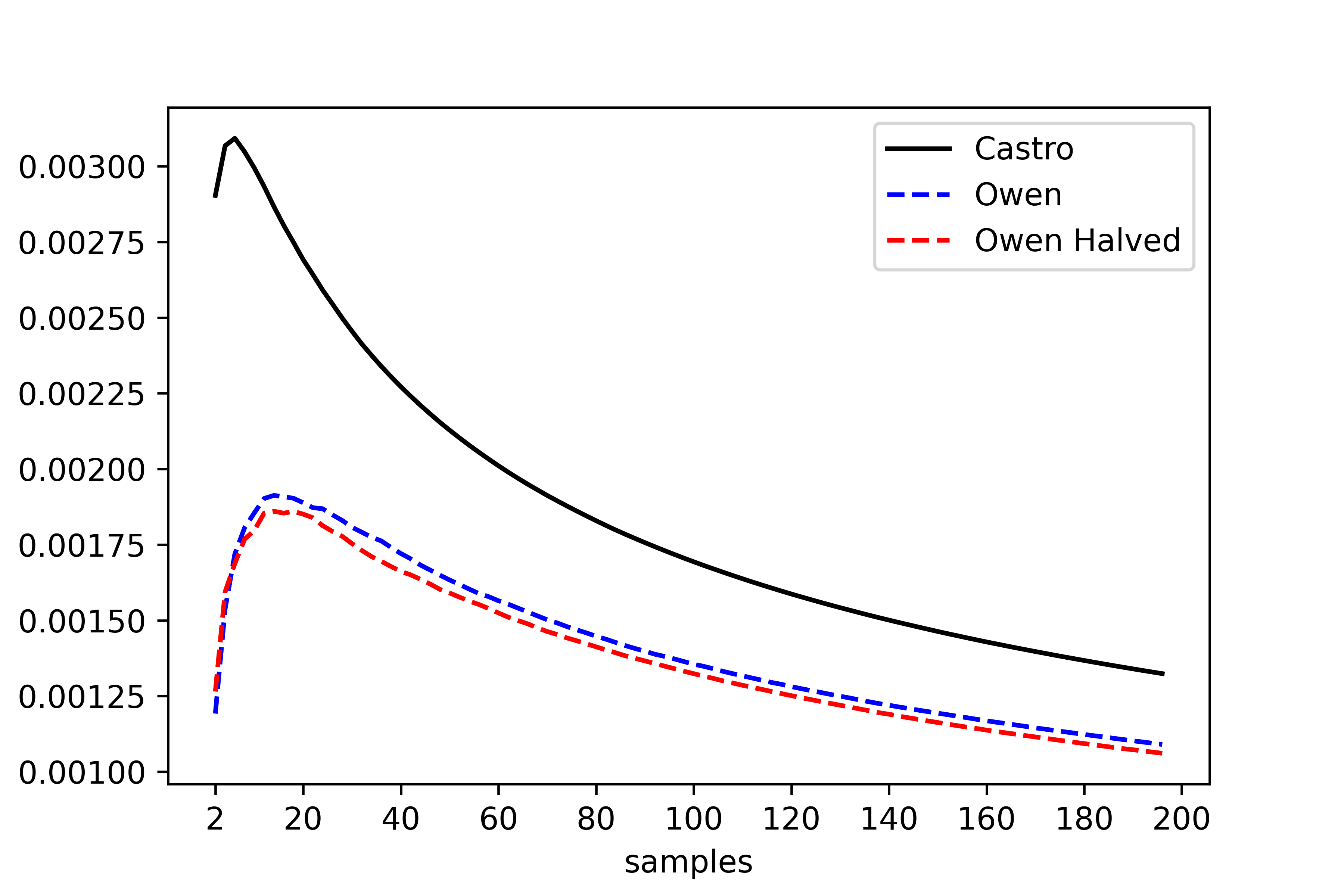}}
    \caption{Variance analysis of the estimators: the left-hand side plot is computed for the CC dataset, and the right-hand one for the MNIST dataset. Both plots are based on 50 examples of the test sets.}
    \label{fig:variance_analysis}
    \vspace{1em}
\end{figure*}

In Figure \ref{fig:variance_analysis}, we present two plots each computed on a different dataset, CC on the left and MNIST on the right. 
The x-axis and y-axis of these plots represent the number of equivalent samples and the sample standard deviation.
As previously mentioned, since all the estimators converge, these sampled standard deviation tend to decrease. 
Here we can observe that as anticipated when discussing the previous experiment, the Castro sampling has a higher variance with respect to the Owen sampling. Moreover, we see that the Halved Owen sampling performs better than the Owen sampling. These results agree with what is observed in the previous experiment.

\begin{figure*}[!ht]
  \vspace{1.5em}
  \centering
  \subfloat[Ground Truth (Castro Sampling with $M_c=5,000$) \label{1a}]{%
       \includegraphics[trim={5cm 0 3cm 0},clip,width=0.98\linewidth]{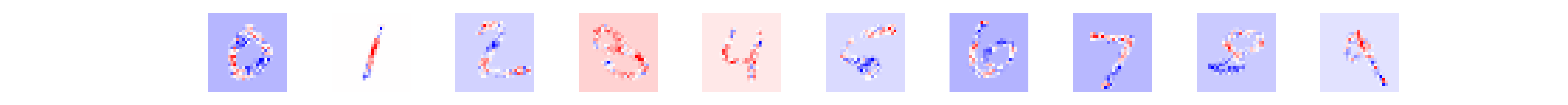}}\\
  \subfloat[Castro Sampling ($M_c=12$)\label{1b}]{%
        \includegraphics[trim={5cm 0 3cm 0},clip,width=0.98\linewidth]{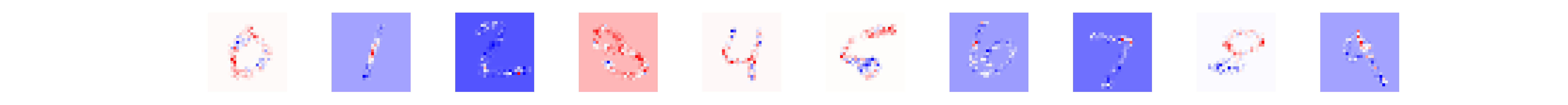}}\\
  \subfloat[Halved Owen Sampling ($M=2$ and $Q=6$)\label{1c}]{%
        \includegraphics[trim={5cm 0 3cm 0},clip,width=0.98\linewidth]{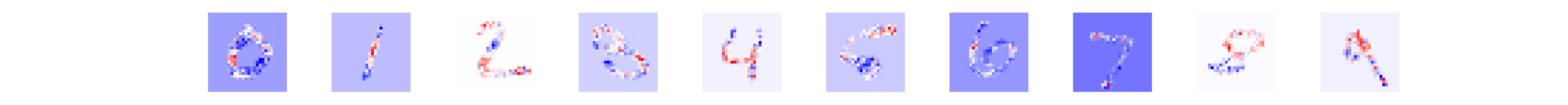}}
  \caption{Saliency maps of 10 randomly selected digits from the test set of the MNIST dataset. We use the top-most row as a reference to compare the results of the following two rows.}
  \label{fig:qualitative_analysis} 
\end{figure*}

Finally, in Figure \ref{fig:qualitative_analysis}, we provide a qualitative analysis of the results of these sampling algorithms. This in order to observe whether the quality of these estimators are perceivable when producing saliency maps -- a common way to show on what the model is basing its prediction. We do this for the MNIST dataset.
In this case we compare the results of the Castro sampling against the Halved Owen sampling with an equivalent number of samples of 12. In order to have a reference, we use as a ground truth the results of the Castro sampling algorithm when computed with a very large number of samples, 5,000 in this case. These images' pixels colored in blue represent the negative contributions and in red the positive ones towards the predicted score. When judging these images we should not focus on their background color because this depends on the maximum and minimum pixel values of the image. In these images we should focus instead on the most saturated pixels, these are in fact those that mostly contribute towards or against the model prediction. 
In these images we can observe that in most of the cases the Owen sampling  identifies better those highly contributing pixels than the Castro sampling.

%% file: conclusion.tex
We have provided a sampling algorithm to efficiently estimate Shapley values that can be also used as a ground truth for comparison purposes, since we have proved the convergence of this algorithm to the exact Shapely values. 
%
The method takes advantage of a variance reduction method and provide more accurate estimations for the Shapley values. 
In all the experiments that we have carried out, our comparison with the Castro algorithm \cite{castro2009polynomial} has been based on equal bases, i.e., we have used the equivalent number of samples. 
In our experiments, MLP models are used to fit with the datasets for classification problems (though regression problems could be considered as well). 
An optimal architecture of the MLPs are found using a cross validation analysis. 
Then we have used these MLPs to test the effectiveness of our results on the corresponding datasets. 
First, we have started with a dataset with small number of features in which we can find the exact measurements of Shapley values. By selecting random instances from the test dataset and calculating the MSE in this case, our algorithms provide clear-cut improvements over the existing sampling algorithm leading to more accurate estimations. While an exact measurement of the Shapley values on datasets with a large number of features is not feasible, we have used an analysis of variance to show that our algorithm provide estimators for which their variance rapidly flattens leading to more accurate estimators. For consistency, the same analysis is also carried out for datasets with small number of features, and the same results are confirmed. 

The experiments and the algorithms  can be still improved in several directions which we leave to future work. More experimental analysis on different datasets could be carried out more complex deep learning architectures than MLPs, since our algorithm could work with any machine learning model. 
The accuracy of our algorithm is controlled by two parameters, in our analysis, we took one of the parameters to be equal to 2. However, more efficient and smart combinations of these parameters might improve the performance of our algorithms, in particular combining this with some statistical analysis, it might be possible to obtain the least number of operations required to reach certain accuracy through finding confidence intervals. 